\documentclass[a4paper]{article}

\usepackage{times}
\usepackage{graphicx} 
\usepackage{subfigure} 

\usepackage{natbib}

\usepackage{algorithm}
\usepackage{algorithmic}

\usepackage{url}            
\usepackage{booktabs}       
\usepackage{amsfonts}       
\usepackage{nicefrac}       
\usepackage{microtype}      
\usepackage{dsfont}

\usepackage{amsmath,amssymb,amsthm}
\usepackage{bbm}
\usepackage{graphicx}
\usepackage{color}
\usepackage{bbm}
\usepackage{multirow}
\usepackage{url}
\usepackage{balance}
\usepackage{mathtools}
\usepackage{arydshln}

\usepackage{pgfplots}
\usepackage{tikz}
\usepackage{standalone}

\newcommand{\hyperprior}{\pi}
\newcommand{\hyperposterior}{\rho}
\newcommand{\hyperpostS}{\rho_S}
\newcommand{\prior}{P}
\newcommand{\posterior}{Q}
\newcommand{\postS}{{Q_S}}
\newcommand{\postvS}{{Q_{v,S}}}
\newcommand{\postvStwo}{{Q_{v,S}^2}}
\newcommand{\postStwo}{{Q_S^2}}

\newcommand{\mv}{\textsc{mv}} 

\newcommand{\KL}{\operatorname{KL}}
\newcommand{\kl}{\operatorname{kl}}

\newcommand{\Ind}[1]{\mathds{1}_{#1}}

\newcommand{\eq}{e_{\Dcal}(Q)}
\newcommand{\dq}{d_{\Dcal}(Q)}
\newcommand{\emv}{e_{\Dcal}^{\mv}(\rho)}
\newcommand{\dmv}{d_{\Dcal}^{\mv}(\rho)}

\newcommand{\emvsl}{e_{S_l}^{\mv}(\rho)}
\newcommand{\dmvsu}{d_{S_u}^{\mv}(\rho)}
\newcommand{\emvs}{e_{S}^{\mv}(\rho)}
\newcommand{\dmvs}{d_{S}^{\mv}(\rho)}

\newcommand{\dmvS}{d_{\Dcal}^{\mv}(\rho_S)}
\newcommand{\emvsS}{e_{S}^{\mv}(\rho_S)}
\newcommand{\dmvsS}{d_{S}^{\mv}(\rho_S)}

\newcommand{\xbf}{{\bf x}}

\newcommand{\D}{{\cal D}}
\newcommand{\Dcal}{{\cal D}}
\newcommand{\Xcal}{{\cal X}}
\newcommand{\Ycal}{{\cal Y}}

\newcommand{\Hcal}{{\cal H}}  
\newcommand{\Vcal}{{\cal V}}  

\newcommand{\concatSVM}{\texttt{Concat}_{\texttt{SVM}}}

\newcommand{\mono}{\texttt{Mono}_v}

\newcommand{\AggregU}{\texttt{Aggreg}_{\texttt{P}}}
\newcommand{\AggregUL}{\texttt{Aggreg}_{\texttt{L}}}

\newcommand{\AggregSVMall}{\texttt{Fusion}_{\texttt{SVM}}^{\texttt{all}}}

\newcommand{\AggregCqall}{\texttt{Fusion}_{\texttt{Cq}}^{\texttt{all}}}

\DeclareMathOperator*{\Esp}{\mathbb{E}} 
\newcommand{\E}[1]{{\displaystyle \Esp_{#1}}}
\newcommand{\Eclap}[1]{\E{\mathclap{#1}}}

\newcommand{\sign}{\operatorname{sign}}

\DeclareMathOperator*{\argmin}{\mathrm{argmin}}

\newcommand{\eg}{{\em e.g.\/}}

\newtheorem{theorem}{Theorem}
\newtheorem{lemma}{Lemma}
\newtheorem{cor}{Corollary}

\usepackage{fancyhdr}
\usepackage{fancybox}
\newcommand{\ver}{3}
\title{PAC-Bayesian Analysis for a two-step Hierarchical Multiview Learning Approach} 
\author{Anil Goyal$^{1,2}$ \and Emilie Morvant$^1$ \and Pascal Germain$^3$ \and  Massih-Reza Amini$^2$ \and 
	$\mbox{}^1$ \small Univ Lyon, UJM-Saint-Etienne, CNRS, Institut d'Optique Graduate School, \\ \small  Laboratoire Hubert Curien UMR 5516, F-42023, Saint-Etienne, France
	\and
	$\mbox{}^2$ \small Univ. Grenoble Alps, Laboratoire d'Informatique de Grenoble, AMA, \\ \small Centre Equation 4, BP 53, F-38041 Grenoble Cedex 9, France
	\and
	$\mbox{}^3$ \small 	D\'epartement d'informatique de l'ENS, \'Ecole Normale Sup\'erieure, \\ \small CNRS, PSL Research University, 75005 Paris, France \\ \small INRIA
}

\begin{document} 

	\maketitle
	
	\begin{abstract} 
		We study a two-level multiview learning with more than two views under the PAC-Bayesian framework. This approach, sometimes referred as late fusion, consists in learning sequentially multiple view-specific classifiers at the first level, and then combining these view-specific classifiers at the second level. Our main theoretical result is a generalization bound on the risk of the majority vote which exhibits a term of diversity in the predictions of the view-specific classifiers. From this result it comes out that controlling the trade-off between diversity and accuracy is a key element for multiview learning, which complements other results in multiview learning. Finally, we experiment our principle on multiview datasets extracted from the Reuters RCV1/RCV2 collection.

	\end{abstract} 
	
	\section{Introduction}
\label{sec:intro}
With the ever-increasing observations produced by more than one source, multiview learning has been expanding over the past decade, spurred by the seminal work of \citet{blum98} on co-training. 
Most of the existing methods try to combine
 multimodal information, either by directly merging the views or by combining models learned from the different views\footnote{The fusion of descriptions, {\it resp.} of models, is sometimes called Early Fusion, {\it resp.} Late Fusion.}~\citep{Early-Late-ACMMultimedia05}, in order to produce a model more reliable for the considered task.  
Our goal is to propose a theoretically grounded criteria to ``correctly'' combine the views.
With this in mind we propose to study multiview learning through the PAC-Bayesian framework (introduced in~\citep{McAllester99}) that allows to derive generalization bounds for models that are expressed as a combination over a set of voters. 
When faced with learning from one view, the PAC-Bayesian theory assumes a prior distribution over the voters involved in the combination, and aims at learning---from the learning sample---a posterior distribution that leads to a well-performing combination expressed as a weighted majority vote.
In this paper we extend the PAC-Bayesian theory to multiview with more than two views.
Concretely, given a set of view-specific classifiers, we define a hierarchy of posterior and prior distributions over the views, such that {\it(i)} for each view $v$, we consider prior $\prior_v$ and posterior $\posterior_v$ distributions over each view-specific voters' set, and {\it (ii)} a prior~$\hyperprior$ and a posterior~$\hyperposterior$ distribution over the set of views (see Figure~\ref{fig:MultiviewHierarchy}), respectively called hyper-prior and hyper-posterior\footnote{Our notion of hyper-prior and hyper-posterior distributions is different than the one proposed for lifelong learning~\citep{PentinaL14}, where they basically consider hyper-prior and hyper-posterior over the set of possible priors: The prior distribution $P$ over the voters' set is viewed as a random variable.}. 
In this way, our proposed approach encompasses the one of \citet{Massih09} that considered uniform distribution to combine the view-specific classifiers' predictions.
Moreover, compared to the PAC-Bayesian work of \citet{SunS14}, we are interested here to the more general and natural case of multiview learning with more than two views. 
Note also that \citet{lecue2014} proposed a non-PAC-Bayesian theoretical analysis of a combination of voters (called \mbox{$Q$-Aggregation}) that is able to take into account  a prior and a posterior distribution but in a single-view setting.
\begin{figure}[t!]
	\label{fig:MultiviewHierarchy}\centering
	\includegraphics[scale=0.4]{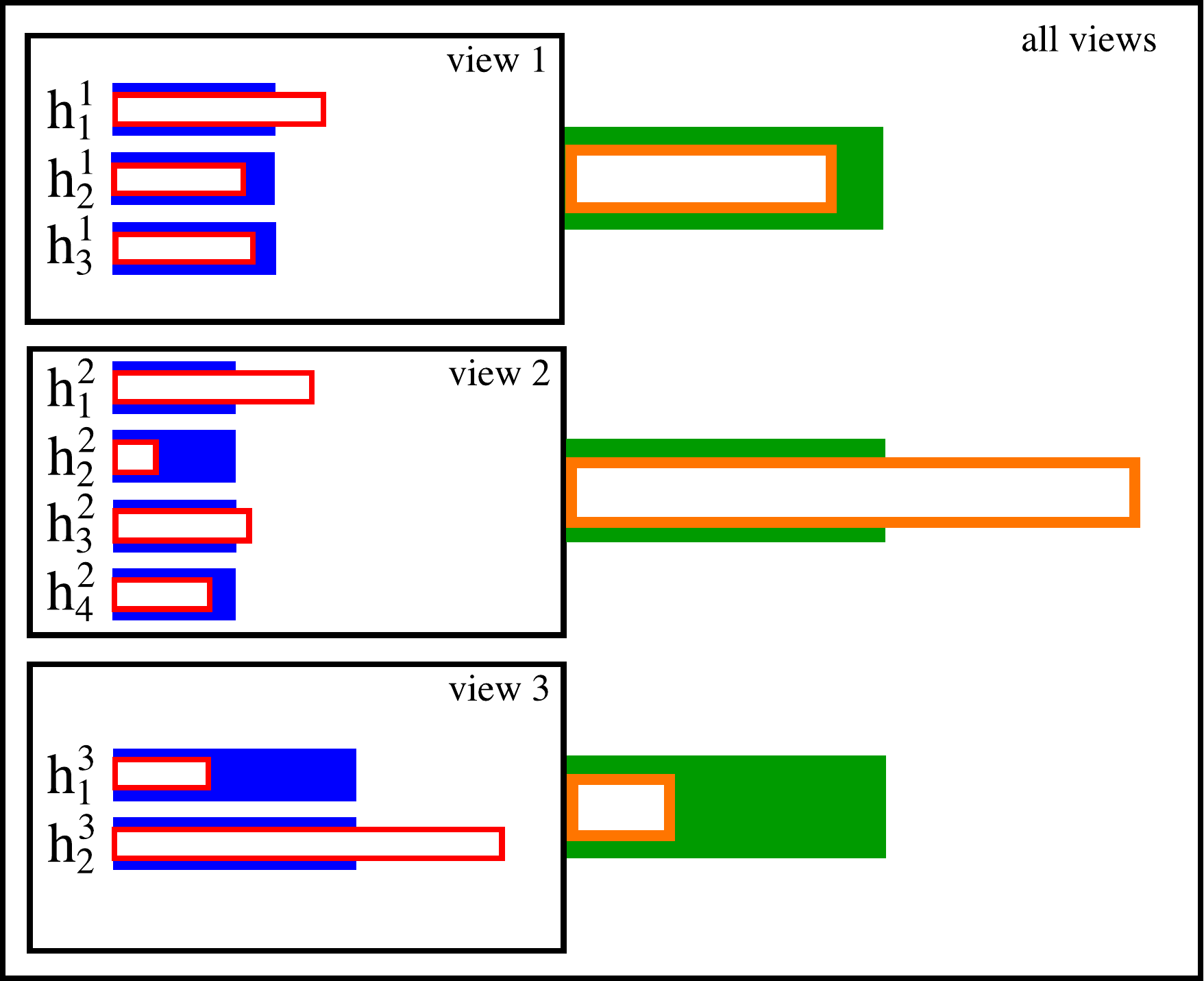}
	\caption{Example of the multiview distributions hierarchy with $3$ views. For all views $v\in\{1,2,3\}$, we have a set of voters  $\Hcal_v=\{h_1^v,\ldots,h_{n_v}^v\}$ on which we consider prior $\prior_v$ view-specific distribution (in blue), and  we consider a hyper-prior~$\hyperprior$ distribution (in green) over the set of $3$ views. The objective is to learn a posterior $\posterior_v$ (in red) view-specific distributions and a hyper-posterior $\hyperposterior$ distribution (in orange) leading to a good model. The length of a rectangle  represents the weight (or probability) assigned to a voter or a view.}
\end{figure}

Our theoretical study also includes a notion of disagreement between all the voters, allowing to take into account a notion of diversity between them which is known as a key element in multiview learning~\citep{Kuncheva,chapelle2006semi,Maillard09,Massih09}. 
Finally, we empirically evaluate a two-level learning approach on the Reuters RCV1/RCV2 corpus to show that our analysis is sound.

In the next section, we recall the general PAC-Bayesian setup, and present PAC-Bayesian \emph{expectation bounds}---while most of the usual PAC-Bayesian bounds are \emph{probabilistic bounds}. 
In Section~\ref{sec:MVPB-Bound}, we then discuss the problem of multiview learning, 
adapting the PAC-Bayesian expectation bounds to the specificity of the two-level multiview approach. In Section~\ref{sec:ComparativeStudy}, we discuss the relation between our analysis and previous works.
Before concluding in Section~\ref{sec:conclu}, we present experimental results obtained on a collection of the Reuters RCV1/RCV2 corpus in Section~\ref{sec:experiments}.

	\section{The Single-View PAC-Bayesian Theorem}
\label{sec:PAC-Bayes}

In this section, we state a {\it new} general mono-view PAC-Bayesian theorem, inspired by the work of \citet{GermainLLMR15}, that we extend to multiview learning in Section~\ref{sec:MVPB-Bound}.
\subsection{Notations and Setting}
We consider binary classification tasks on data drawn from a fixed yet unknown distribution $\D$ over $\Xcal \times \Ycal$, where $\mathcal{X}  \subseteq  \mathbb{R}^d$ is a \mbox{$d$-dimensional} input space and $\mathcal{Y} = \{-1,+1\}$ the label/output set. 
A learning algorithm is provided with a training sample of $m$ examples denoted by $S=\{ (x_i,y_i ) \}_{i=1}^{m} \in  (\Xcal \times \Ycal)^m$, that is assumed to be independently and identically distributed ({\it i.i.d.}) according to $\D$. 
The notation $\D^m$ stands for the distribution of such a \mbox{$m$-sample}, and $\D_\Xcal$ for the marginal distribution on $\Xcal$. 
We consider a set $\Hcal$ of classifiers or voters such that $\forall h\in\Hcal,\ h:\Xcal \to\Ycal$. 
In addition, PAC-Bayesian approach requires a prior distribution $\prior$ over $\Hcal$ that models {\it a priori} belief on the voters from $\Hcal$ before the observation of the learning sample $S$.
Given $S\sim\D^m$, the learner objective is then to find 
a posterior distribution $\posterior$ over $\Hcal$ leading to an accurate \mbox{$\posterior$-weighted} majority vote $B_Q(x)$ defined as
\begin{align*}
B_Q(x) = \sign \left[ \E{ h \sim \posterior} h(x)\right].
\end{align*} 
In other words, one wants to learn $\posterior$ over $\Hcal$ such that it minimizes the true risk $R_{\D}(B_{Q})$ of $B_Q(x)$: 
 $$R_{\mathcal{D}}(B_{Q}) = \E{(x,y) \sim \mathcal{D}} \Ind{[B_{Q}(x) \neq y]}\,,$$
where $\Ind{[\pi]} =1$ if predicate $\pi$ holds, and $0$ otherwise.
However, a PAC-Bayesian generalization bound does not directly focus on the risk of the deterministic \mbox{$\posterior$-weighted} majority vote $B_{Q}$. 
Instead, it upper-bounds the risk of the stochastic Gibbs classifier $G_{Q}$, which predicts the label of an example $x$ by drawing $h$ from $\Hcal$ according to the posterior distribution $\posterior$ and predicts $h(x)$.
Therefore, the true risk $R_{D}(G_Q)$ of the Gibbs classifier on a data distribution $\D$, and its empirical risk $R_{S}(G_Q)$ estimated on a sample $S \sim  \D^m$  are respectively given by
\begin{align*}
 R_{\D}(G_Q) \ &=\  \E{(x,y) \sim \mathcal{D}} \ \E{h \sim \posterior} \Ind{[h(x) \neq y]}\,,\\
\mbox{ and }\quad
R_{S}(G_Q) \ &=\ \frac{1}{m} \sum_{i=1}^m \E{h \sim \posterior} \Ind{[h(x_i) \neq y_i]}\,.
\end{align*}
The above Gibbs classifier is closely related to the \mbox{$\posterior$-weighted} majority vote $B_{\posterior}$.
Indeed, if $B_{\posterior}$ misclassifies $x \in  \Xcal$, then at least half of the classifiers (under measure $\posterior$) make an error on $x$. 
Therefore, we have 
\begin{align}
\label{eq:gibbsrelation}
R_{\mathcal{D}}(B_Q) \leq 2R_{\mathcal{D}}(G_Q).
\end{align}
Thus, an upper bound on $R_\D(G_\posterior)$ gives rise to an upper bound on $R_\D(B_\posterior)$.
Other tighter relations exist \citep{LangfordS02,Lacasse06,GermainLLMR15}, such as the so-called \mbox{C-Bound} \citep{Lacasse06} that involves the \textit{expected disagreement} $\dq$ between all the pair of voters, and that can be expressed as follows (when $R_{\Dcal}(G_Q)\leq \frac12$):
\begin{align}
\label{eq:cbound}
&R_{\mathcal{D}}(B_{Q}) \ \leq\  1 - \frac{\displaystyle \left(1-2R_{\Dcal}(G_Q)\right)^2}{\displaystyle 1-2 \dq}\,,\\
\nonumber \mbox{ where}\quad &\dq  = \E{x \sim \D_\Xcal}\, \E{(h,h') \sim \posterior^2} \Ind{[h(x) {\neq} h'(x)]}\,.
\end{align}
Moreover, \citet{GermainLLMR15} have shown that the Gibbs classifier's risk can be rewritten in terms of $\dq$ and  \textit{expected joint error} $\eq$ between all the pair of voters as
\begin{align}
\label{eq:ed}
&R_{\mathcal{D}}(G_Q) \ =\  \frac{1}{2} \dq+\eq\,,\\
\nonumber\mbox{where}\quad &
  \eq\  =\ \E{(x,y) \sim \mathcal{D}}\  \E{(h,h') \sim Q^2}\ \Ind{[h(x) {\neq} y]}\, \Ind{[h'(x){\neq} y]}\,.
\end{align}
It is worth noting that from multiview learning standpoint where the notion of diversity among voters is known to be important~\citep{Massih09,Maillard09,SunS14,AtreyHEK10,Kuncheva}, Equations~\eqref{eq:cbound} and~\eqref{eq:ed} directly capture the trade-off between diversity and accuracy. 
Indeed, $\dq$ involves the diversity between voters~\citep{MorvantHA14}, while $\eq$ takes into account the errors.
Note that the principle of controlling the trade-off between diversity and accuracy through the C-bound of Equation~\eqref{eq:cbound} has been exploited by \citet{MinCQ} and \citet{cqboost} to derive well-performing PAC-Bayesian algorithms that aims at minimizing it. For our experiments in Section~\ref{sec:expe}, we make use of CqBoost~\citep{cqboost}---one of these algorithms---for multiview learning.
\\
Last but not least, PAC-Bayesian generalization bounds take into account the given prior distribution $\prior$ on $\Hcal$ through the Kullback-Leibler divergence between the learned posterior distribution $\posterior$ and  $\prior$: 
\begin{align*}
\KL(\posterior \| \prior)\ =\ \E{h \sim \posterior} \ln \frac{\posterior(h)}{\prior(h)}\,.
\end{align*}

\subsection{A New PAC-Bayesian Theorem as an Expected Risk Bound}
In the following we introduce a new variation of the general PAC-Bayesian theorem of \citet{GermainLLM09,GermainLLMR15}; it takes the form of an upper bound on the ``deviation'' between the true risk $ R_{\D}(G_Q)$  and empirical risk $ R_{S}(G_Q) $ of the Gibbs classifier, according to a convex function $D {:} [0 , 1]{\times} [0, 1]{\to}\mathbb{R}$. 
While most of the PAC-Bayesian bounds are probabilistic bounds, we state here an \emph{expected risk bound}.
More specifically, Theorem~\ref{theo:PB} below is a tool to upper-bound $\mathbb E_{{S\sim\D^m}} R_{\D}(G_\postS)$---where $\postS$ is the posterior distribution outputted by a given learning algorithm after observing the learning sample $S$---while PAC-Bayes usually bounds $R_{\D}(G_Q)$ uniformly for all distribution $\posterior$, but with high probability over the draw of $S \sim \D^m$. 
Since by definition posterior distributions are data dependent, this different point of view on PAC-Bayesian analysis has the advantage to involve an expectation over all the possible learning samples (of a given size) in bounds itself.
\begin{theorem}
	\label{theo:PB}
	For any distribution $\mathcal{D}$ on $\Xcal \times \Ycal$, for any set of voters $\Hcal$, for any prior distribution $\prior$ on $\mathcal{H}$, for any convex function $D : [0 , 1]  \times  [0, 1] \to \mathbb{R}$, we have
	\begin{align*}
	 D \left(\ \E{{S\sim\D^m}} R_{S}(G_\postS)  , \E{S\sim\D^m}R_{\D}(G_\postS)   \right) 
	\leq&\    \frac{1}{m}   \left[ \ \ \ \E{\mathclap{S\sim\D^m}}\ \KL(\postS \| \prior)   +  \ln \bigg( \E{S \sim \D^m}\, \E{h \sim \prior}  e^{m\,D\left(R_S(h), R_\D(h) \right)}   \bigg)   \right] ,   
	\end{align*}
	where $R_\D(h)$ and $R_S(h)$ are respectively the true and the empirical risks of individual voters. 
\end{theorem}
Similarly to \citet{GermainLLM09,GermainLLMR15}, by selecting a well-suited deviation function~$D$  and by upper-bounding 
$\mathbb E_{S}\, \mathbb E_{h}  e^{m\,D(R_S(h),R_\D(h))}$,
we can prove the \emph{expected bound} counterparts of the classical PAC-Bayesian theorems of~\citet{McAllester99,Seeger02,catoni2007pac}. 
The proof presented below borrows the straightforward proof technique of~\citet{begin-16}.
Interestingly, this approach highlights that the expectation bounds are obtained simply by replacing the \emph{Markov inequality}  by the \emph{Jensen inequality} (respectively Theorems~\ref{theo:markov} and~\ref{theo:jensen}, in Appendix).
\begin{proof}[Proof of Theorem~\ref{theo:PB}] The last three inequalities below are  obtained by applying Jensen's inequality on the convex function $D$, the change of measure inequality \citep[as stated by][Lemma~3]{begin-16}, and Jensen's inequality on the concave function $\ln$.
\allowdisplaybreaks[1]
\begin{align*}
	 m D \left( \E{{S\sim\D^m}} R_{S}(G_\postS)  , \E{S\sim\D^m}R_{\D}(G_\postS)   \right) \
	 & =\ m D \left( \E{{S\sim\D^m}} \E{h\sim\postS} R_{S}(h)  , \E{S\sim\D^m} \E{h\sim\postS}R_{\D}(h)   \right) \\
	 & \leq\   \E{{S\sim\D^m}} \E{h\sim\postS}m D \left( R_{S}(h)  , R_{\D}(h)   \right) \\
         &\leq\ \E{{S\sim\D^m}}\left[ \KL(\postS \| \prior)   +  \ln \bigg(\E{h \sim \prior}  e^{m\,D\left(R_S(h), R_\D(h) \right)}   \bigg) \right] \\
	 &\leq\  \E{{S\sim\D^m}}   \KL(\postS \| \prior)    +  \ln \bigg(\E{{S\sim\D^m}} \E{h \sim \prior}  e^{m\,D\left(R_S(h), R_\D(h) \right)}   \bigg)\,.
\end{align*}
\end{proof}
Since the C-bound of Equation~\eqref{eq:cbound} involves the expected disagreement $\dq$, we also derive below the expected bound that upper-bounds the deviation between $\Esp_{{S\sim\D^m}} d_S(Q_S)$  and $\Esp_{S\sim\D^m}d_\D(Q_S)$ under a convex function $D$.
Theorem~\ref{theo:PBd} can be seen as the \emph{expectation} version of probabilistic bounds over $d_S(Q_S)$ proposed by \citet{Lacasse06,GermainLLMR15}.
\begin{theorem}
	\label{theo:PBd}
	For any distribution $\mathcal{D}$ on $\Xcal \times \Ycal$, for any set of voters $\mathcal{H}$, for any prior distribution $\prior$ on $\mathcal{H}$, for any convex function $D : [0 , 1]  \times  [0, 1] \to \mathbb{R}$, we have
	\begin{align*}
	& D \left( \E{{S\sim\D^m}} d_S(Q_S)  , \E{S\sim\D^m}d_\D(Q_S)   \right) \ \leq\    \frac{2}{m}  \bigg[ \ \ \ \Eclap{S\sim\D^m}\KL(\postS \| \prior)   +  \ln \sqrt{   \E{S \sim \D^m}\E{(h,h') \sim \prior^2} \hspace{-5mm} e^{mD\left(d_S(h,h'), d_\D(h,h') \right)}  } \bigg] ,   
	\end{align*}
	where  $d_\D(h,h') = \Esp_{x \sim \D_\Xcal}\, \Ind{[h(x) {\neq} h'(x)]}$
	is the disagreement of voters $h$ and $h'$ on the distribution $\D$, and $d_S(h,h')$ is its empirical  counterpart.
\end{theorem}
\begin{proof}
First, we apply the exact same steps as in the proof of Theorem~\ref{theo:PB}:
\begin{align*}
& m D \left( \E{{S\sim\D^m}} d_S(Q_S)  , \E{S\sim\D^m}d_\D(Q_S)   \right)  \ =\ m D \left( \E{{S\sim\D^m}}\ \E{(h,h')\sim\postStwo}\ d_S(h,h')  , \E{S\sim\D^m}\ \E{(h,h')\sim\postStwo}\ d_\D(h,h')   \right) 
\\
\vdots \ & \\
\leq& \ \ \E{S\sim\D^m}  \KL(\postStwo \| \prior^2)   +  \ln     \E{S \sim \D^m}\E{(h,h') \sim \prior^2} \hspace{-5mm} e^{mD\left(d_S(h,h'), d_\D(h,h') \right)} .
\end{align*}
Then, we use the fact that  
$\KL(\postStwo \| \prior^2)  = 2\KL(\postS \| \prior) $ \citep[see][Theorem~25]{GermainLLMR15}.
\end{proof}

In the following we  provide an extension of this PAC-Bayesian framework to multiview learning with more than two views.

	\section{Multiview PAC-Bayesian Approach} 
\label{sec:MVPB-Bound}
\subsection{Notations and Setting}
We consider binary classification problems where the multiview observations $\xbf  =  (x^1,\ldots,x^V)$ belong to a multiview input set
\mbox{${\cal X}  =  \mathcal{X}_1\times  \ldots \times\mathcal{X}_V$}, where $V\geq 2$ is the number of views of not-necessarily the same dimension.
We denote $\Vcal$ the set of the $V$ views.
In binary classification, we assume that examples are pairs $(\xbf, y)$, with $y\in{\cal Y} = \{-1,+1\}$, drawn according to an unknown distribution $\D$ over $\Xcal  \times  \mathcal{Y}$.
To model the two-level multiview approach, we follow the next setting.
For each view $v \in \Vcal$, we consider a view-specific set $\Hcal_v$ of voters $h  :  \mathcal{X}_v \to {\cal Y}$, and a  prior distribution $\prior_v$ on $\Hcal_v$. Given a \emph{hyper-prior} distribution $\hyperprior$ over the views $\Vcal$, and a multiview learning sample $S = \{(\xbf_i,y_i)\}_{i=1}^m \sim (\D)^m$, our PAC-Bayesian learner \mbox{objective} is twofold: {\it (i)} finding a posterior distribution $\posterior_v$ over $\mathcal{H}_v$  for all views $v \in \Vcal$; {\it (ii)} finding a \emph{hyper-posterior} distribution $\hyperposterior$ on the set of views $\Vcal$. 
This hierarchy of distributions is illustrated by Figure~\ref{fig:MultiviewHierarchy}. 
The learned distributions express a multiview weighted majority vote $B_{\hyperposterior}^{\mv}$ defined as 
$$B_{\hyperposterior}^{\mv}(\mathbf{x}) = \sign \left[\E{ v \sim \hyperposterior}\  \E{ h \sim \posterior_v} h(x^v) \right].$$
Thus, the learner aims at constructing the posterior and hyper-posterior distributions that minimize the true risk $R_{\mathcal{D}}(B_{\hyperposterior}^{\mv})$ of the multiview weighted majority vote:
$$ R_{\mathcal{D}}(B_{\hyperposterior}^{\mv}) = \E{(\mathbf{x},y) \sim \mathcal{D}} \Ind{[B_{\hyperposterior}^{\mv}(\mathbf{x}) \neq y]}.
$$
As pointed out in Section~\ref{sec:PAC-Bayes}, the PAC-Bayesian approach deals with the risk of the stochastic Gibbs classifier $G_{\hyperposterior}^{\mv}$ defined as follows in our multiview setting, and that can be rewritten in terms of \textit{expected disagreement} $\dmv$ and \textit{expected joint error} $\emv$:
\begin{align}
\nonumber 
 R_{\mathcal{D}}(G_{\hyperposterior}^{\mv})  \ &=\  \E{(\xbf,y) \sim \D} \ \E{ v \sim \hyperposterior} \ \E{h \sim \posterior_v} \Ind{[h(x^v) \neq y]}\\
\label{eq:SplitGibbsRisk}
 &= \ \tfrac{1}{2}\, \dmv + \emv\,,\\
\nonumber\text{where}\quad 
\dmv \  &= \  \Esp_{\xbf \sim \mathcal{D}_{\mathcal{X}}} \Esp_ {v \sim \hyperposterior} \Esp_{v' \sim \hyperposterior}   \Esp_{h \sim \posterior_v}  \Esp_{h' \sim \posterior_{v'}}   \Ind{[ h(x^v) {\ne} h'(x^{v'})]}, \\
\nonumber\text{and}\quad  \emv   \  &= \      \E{(\xbf,y) \sim \mathcal{D}} \Esp_ {v \sim \hyperposterior} \Esp_{v' \sim \hyperposterior}   \Esp_{h \sim \posterior_v}   \Esp_{h' \sim \posterior_{v'}}   \Ind{[  h(x^v) {\ne} y ]}  \Ind{[ h'(x^{v'}) {\ne} y ]}.
\end{align}
Obviously, the empirical counterpart of the Gibbs classifier's risk  $R_{\mathcal{D}}(G_{\hyperposterior}^{\mv})$ is 
\begin{align*}
R_{S}(G_{\hyperposterior}^{\mv}) &= \frac{1}{m} \sum_{i=1}^m \E{ v \sim \hyperposterior} \  \E{ h \sim \posterior_v} \Ind{[h(x_i^v) \neq y_i]}\\
&= \frac{1}{2}\dmvs + \emvs\,,
\end{align*}
where $\dmvs$ and $\emvs$  are respectively the empirical estimations of $\dmv$ and $\emv$ on the learning sample $S$.
As in the single-view PAC-Bayesian setting, the multiview weighted majority vote $B_{\hyperposterior}^{\mv}$ is closely related to the stochastic multiview Gibbs classifier $G_{\hyperposterior}^{\mv}$, and a generalization bound for $G_{\hyperposterior}^{\mv}$ gives rise to a generalization bound for $B_{\hyperposterior}^{\mv}$. Indeed, it is easy to show that $R_{\mathcal{D}}(B_\hyperposterior^{\mv})\leq 2R_{\mathcal{D}}(G_{\hyperposterior}^{\mv})$, meaning that an upper bound over $R_{\mathcal{D}}(G_{\hyperposterior}^{\mv})$ gives an upper bound for the majority vote.
Moreover the C-Bound of Equation~\eqref{eq:cbound} can be extended to our multiview setting by Lemma~\ref{lem:mv-cbound} below.
Equation~\eqref{eq:cbound_multiview} is a straightforward generalization of the single-view C-bound of Equation~\eqref{eq:cbound}. 
Afterward, Equation~\eqref{eq:cbound_multiview_2} is obtained by rewriting $R_{\mathcal{D}}(G_{\hyperposterior}^{\mv})$ as the $\hyperposterior$-average of the risk associated to each view, and lower-bounding $\dmv$  by the $\hyperposterior$-average of the disagreement associated to each view.
\begin{lemma} \label{lem:mv-cbound}
	Let $V \geq 2$ be the number of views. 
	For all posterior $\{\posterior_v\}_{v=1}^V$ and hyper-posterior $\hyperposterior$ distribution, if $R_{\Dcal}(G_{\hyperposterior}^{\textsc{\mv}}) <\frac12$, then we have
\begin{eqnarray}
\label{eq:cbound_multiview}
R_{\mathcal{D}}(B_\hyperposterior^{\mv})
&\leq &
1- \frac{\displaystyle  \big(1-2R_{\Dcal}(G_{\hyperposterior}^{\textsc{\mv}})\big)^2}{\displaystyle 1-2 \dmv} \label{eq:cbound_multiview_4}\\
&\ \leq\  & 
1- \frac{  \Big(1-2  \Esp_ {v \sim \hyperposterior} R_{\Dcal}(G_{Q_v})\Big)^2}{ 1-2 \Esp_ {v \sim \hyperposterior}  d_{\Dcal}(Q_v) } \label{eq:cbound_multiview_2}\,.
\end{eqnarray}
\end{lemma}
\begin{proof} Equation~\eqref{eq:cbound_multiview_4} follows from the Cantelli-Chebyshev's inequality (Theorem~\ref{theo:chebyshev}, in Appendix). 
To prove Equation~\eqref{eq:cbound_multiview_2}, we first notice that in the binary setting where $y\in\{-1,1\}$ and $h:\Xcal\to\{-1,1\}$, we have 
$\Ind{[h(x^v) \neq y]}  =  \frac12 (1-y\,h(x^v))$, and
\begin{align}
R_{\mathcal{D}}(G_{\rho}^\mv) 
  \nonumber 
\ &=\ \E{(\mathbf{x},y) \sim \mathcal{D}} \ \E{ v \sim \hyperposterior} \ \E{h \sim \posterior_v} \Ind{[h(x^v) \neq y]} \nonumber  \\
& 
 =\ \frac{1}{2}\bigg(1- \E{(\xbf,y) \sim \mathcal{D}} \  \E{v \sim \hyperposterior} \ \E{h \sim \posterior_v} y\,h(x^v) \bigg) \nonumber  \\
  & 
  =\  
 \E{v \sim \hyperposterior} R_{\mathcal{D}}(G_{Q^v})\,.\nonumber 
\end{align}
Moreover, we have
\begin{align}
\dmv 
 \nonumber
\ &=\  \Esp_{\xbf \sim \mathcal{D}_{\mathcal{X}}} \Esp_ {v \sim \hyperposterior} \Esp_{v' \sim \hyperposterior}   \Esp_{h \sim \posterior_v}  \Esp_{h' \sim \posterior_{v'}} \Ind{[ h(x^v) {\ne} h'(x^{v'})]} \\
&=\ \frac{1}{2} \bigg( 1- \Esp_{\xbf \sim \mathcal{D}_{\mathcal{X}}} \Esp_ {v \sim \hyperposterior} \Esp_{v' \sim \hyperposterior}   \Esp_{h \sim \posterior_v}  \Esp_{h \sim \posterior_{v'}}  h(x^v) \times h'(x^{v'})  \bigg) \nonumber \\
\nonumber&=\ \frac{1}{2} \bigg( 1- \E{\xbf \sim \mathcal{D}_{\mathcal{X}}}\bigg[  \E{v \sim \hyperposterior} \, \E{h \sim \posterior_v} h(x^v) \bigg]^2  \bigg)  \,.
\end{align}
 From Jensen's inequality (Theorem~\ref{theo:jensen}, in Appendix) it comes
\begin{align*}
\dmv\ 
&\geq\ \frac{1}{2} \bigg( 1- \E{\xbf \sim \mathcal{D}_{\mathcal{X}}} \,  \E{v \sim \hyperposterior}\bigg[  \E{h \sim \posterior_v} h(x^v) \bigg]^2  \bigg) \\
& = \  \E{v \sim \hyperposterior} \Bigg[\frac{1}{2} \bigg( 1- \E{\xbf \sim \mathcal{D}_{\mathcal{X}}}\bigg[ \E{h \sim \posterior_v} h(x^v) \bigg]^2  \bigg) \Bigg] 
\\
&
 =\ \E{v \sim \hyperposterior}  d_{\Dcal}(Q_v)  \,.
\end{align*}
By replacing $R_{\mathcal{D}}(G_{\rho}^\mv)$ and $\dmv$ in  Equation~\eqref{eq:cbound_multiview_4}, we obtain
\begin{align*}
1 -  \frac{\displaystyle  \big(1-2R_{\Dcal}(G_{\hyperposterior}^{\textsc{\mv}})\big)^2}{\displaystyle 1-2 \dmv}
\ \leq \
1 -  \frac{  \Big(1-2  \Esp_ {v \sim \hyperposterior} R_{\Dcal}(G_{Q^v})\Big)^2}{ 1-2 \Esp_ {v \sim \hyperposterior}  d_{\Dcal}(Q_v) } \,.
\end{align*}
\vspace{-10mm}

\end{proof}
Similarly than for the mono-view setting, Equations~\eqref{eq:SplitGibbsRisk} and~\eqref{eq:cbound_multiview_4} suggest that a good trade-off between the risk of the Gibbs classifier $G_{\hyperposterior}^{\textsc{\mv}}$ and the disagreement $\dmv$ between pairs of voters will lead to a well-performing majority vote.
Equation~\eqref{eq:cbound_multiview_2}  exhibits the role of diversity among the views thanks to the disagreement's expectation over the views $\Esp_ {v \sim \hyperposterior}  d_{\Dcal}(Q_v)$.

\subsection{General Multiview PAC-Bayesian Theorems}
\label{sec:theorems}
Now we state our general PAC-Bayesian theorem suitable for the above multiview learning setting with a two-level hierarchy of distributions over views (or voters).
A key step in PAC-Bayesian proofs is the use of a \textit{change of measure inequality}~\citep{McAllester03}, based on the Donsker-Varadhan inequality~\citep{donsker1975}.
Lemma~\ref{lem:change} below extends this tool to our multiview setting.
\begin{lemma} 
	\label{lem:change}
	For any set of priors $\{\prior_v\}_{v=1}^V$ and any set of posteriors $\{\posterior_v\}_{v=1}^V$, for any hyper-prior distribution $\hyperprior$ on views $\Vcal$ and hyper-posterior distribution $\hyperposterior$ on $\Vcal$, 
	and for any measurable function $\phi  :  \mathcal{H}_v  \to  \mathbb{R}$, we have
	\begin{align*}
	& \E{ v \sim \hyperposterior}\  \E{ h \sim \posterior_v} \phi (h)  
	 \leq\  \E{ v \sim \hyperposterior}   \KL(\posterior_v \| \prior_v)  +  \KL(\hyperposterior \| \hyperprior)  +    \ln  \left(  \E{ v \sim \hyperprior}\   \E{ h \sim \prior_v}   e^{\phi (h)}  \right). 
	\end{align*}
\end{lemma}
\begin{proof} We have
\begin{align*}
\nonumber  \E{ v \sim \hyperposterior} \,  \E{ h \sim \posterior_v} \phi (h)\ &=\ \E{ v \sim \hyperposterior} \  \E{ h \sim \posterior_v} \ln e^{\phi (h)} \\
\nonumber   &=\ \E{ v \sim \hyperposterior} \  \E{ h \sim \posterior_v} \ln \bigg( \frac{\posterior_v(h)}{\prior_v(h)} \frac{\prior_v(h)}{\posterior_v(h)} e^{\phi (h)} \bigg)  \\
\nonumber &=\ \E{ v \sim \hyperposterior}\  \bigg[ \E{ h \sim \posterior_v} \ln \bigg( \frac{\posterior_v(h)}{\prior_v(h)} \bigg) +  \E{ h \sim \posterior_v} \ln \bigg( \frac{\prior_v(h)}{\posterior_v(h)} e^{\phi (h)} \bigg) \bigg].
\end{align*}
According to the Kullback-Leibler definition, we have
\begin{align*}
\nonumber \E{ v \sim \hyperposterior} \,  \E{ h \sim \posterior_v} \phi (h) \ =\   \E{ v \sim \hyperposterior} \bigg[  \KL(\posterior_v \| \prior_v) +     \E{ h \sim \posterior_v}   \ln  \bigg( 
\frac{\prior_v(h)}{\posterior_v(h)} e^{\phi (h)}  \bigg)  
\bigg].
\end{align*}
By applying Jensen's inequality (Theorem~\ref{theo:jensen}, in Appendix) on the concave function $\ln$, we have
\begin{align*}
\nonumber  \E{ v \sim \hyperposterior} \,  \E{ h \sim \posterior_v} \phi (h)  \ &\le \ \E{ v \sim \hyperposterior} \  \bigg[ \KL(\posterior_v \| \prior_v) +   \ln \bigg( \E{ h \sim \prior_v} e^{\phi (h)} \bigg) \bigg] \\
\nonumber  &= \ \E{ v \sim \hyperposterior} \KL(\posterior_v \| \prior_v) + \E{ v \sim \hyperposterior} \ln \bigg( \frac{\hyperposterior(v)}{\hyperprior(v)} \frac{\hyperprior(v)}{\hyperposterior(v)} \E{ h \sim \prior_v} e^{\phi (h)} \bigg) \\
\nonumber  &= \ \E{ v \sim \hyperposterior}  \KL(\posterior_v \| \prior_v) + \KL(\hyperposterior \| \hyperprior) 
+    \E{ v \sim \hyperposterior}    \ln  \bigg( \frac{\hyperprior(v)}{\hyperposterior(v)}  \E{ h \sim \prior_v} e^{\phi (h)}  \bigg).
\end{align*}
Finally, we apply again the Jensen inequality (Theorem~\ref{theo:jensen}) on $\ln$ to obtain the lemma.
\end{proof}

Based on Lemma~\ref{lem:change}, the following theorem can be seen as a generalization of Theorem~\ref{theo:PB} to multiview.
Note that we still rely on a general convex function $D : [0,1]  \times  [0,1] \to \mathbb{R}$, that measures the ``deviation'' between the empirical disagreement/joint error and the true risk of the Gibbs classifier.
\begin{theorem}
	\label{theo:MVPB1}
	Let $V \geq 2$ be the number of views. 
	For any distribution $\D$ on $\Xcal \times \Ycal$, for any set of prior distributions $\{\prior_v\}_{v=1}^V$, for any hyper-prior distribution $\hyperprior$ over $\Vcal$, for any convex function $D : [0,1]  \times  [0,1]  \to \mathbb{R}$,
	we have
	\begin{align*}
	 D\Big( \tfrac12   \E{ S \sim \D^m}    \dmvsS  +      \E{ S \sim \D^m}  \emvsS, \E{ S \sim \D^m}    R_{\mathcal{D}}(G_{\hyperpostS}^{\mv})\Big)  
  \leq   \frac{1}{m} \bigg[  \E{ S \sim \D^m}\ \E{ v \sim \hyperpostS}     \KL(\postvS\| \prior_v)&\\
	+   \ \E{S \sim \D^m} \  \KL(\hyperpostS \| \hyperprior)
	 + \ln \left( \E{ S \sim \D^m} \,  \E{ v \sim \hyperprior} \,  \E{ h \sim \prior_v}  e^{ m D\left(R_S(h) , R_\D(h)\right)}  \right)   \bigg].&
	\end{align*}
\end{theorem}
\begin{proof}
We follow the same steps as in Theorem~\ref{theo:PB} proof.
\begin{align*}
& m D\Big(   \E{ S \sim \D^m} R_S(G_{\hyperpostS}^{\mv})	, \E{ S \sim \D^m} R_{\mathcal{D}}(G_{\hyperpostS}^{\mv})\Big)\\ 
=\ & m D\Big(   \E{ S \sim \D^m}  \E{ v \sim \hyperpostS}  \E{ h \sim \postvS} R_S(h)	, \E{ S \sim \D^m}  \E{ v \sim \hyperpostS}  \E{ h \sim \postvS}  R_{\mathcal{D}}(h)\Big)\\ 
\leq\ &  \E{{S\sim\D^m}} \E{ v \sim \hyperpostS}  \E{ h \sim \postvS}m  D \left( R_{S}(h)  , R_{\D}(h)   \right) \\
\leq\ & \E{{S\sim\D^m}} \bigg[ 
\E{ v \sim \hyperpostS} \,  \KL(\postvS\| \prior_v)
+   \KL(\hyperpostS \| \hyperprior) 
+ \ln \left(  \E{ v \sim \hyperprior} \,  \E{ h \sim \prior_v}  e^{ m D\left(R_S(h) , R_\D(h)\right)}  \right)   \bigg] \,,
\end{align*}
where the last inequality is obtained using Lemma~\ref{lem:change}. After distributing the expectation of $S\sim\D^m$, the final statement follows from Jensen's inequality (Theorem~\ref{theo:jensen})
\begin{align*}
 \E{{S\sim\D^m}} \ln \left(  \E{ v \sim \hyperprior} \,  \E{ h \sim \prior_v}  e^{ m D\left(R_S(h) , R_\D(h)\right)}  \right) \ \leq \ \ln \left(  \E{{S\sim\D^m}} \,\E{ v \sim \hyperprior} \,  \E{ h \sim \prior_v}  e^{ m D\left(R_S(h) , R_\D(h)\right)}  \right),
\end{align*}
and from Equation~\eqref{eq:ed}: $R_S(G_{\hyperpostS}^{\mv}) =  \tfrac12 \dmvsS + \emvsS$.
\end{proof}

It is interesting to compare this generalization bound to Theorem~\ref{theo:PB}.
The main difference relies on the introduction of view-specific prior and posterior distributions, which mainly leads to an additional term $\mathbf{E}_{ v \sim \hyperposterior} \KL(\posterior_v \| \prior_v)$, expressed as the expectation of the view-specific Kullback-Leibler divergence term over the views $\Vcal$ according to the hyper-posterior distribution $\hyperposterior$. 
We also introduce the empirical disagreement allowing us to directly highlight the presence of the diversity between voters and between views.
As Theorem~\ref{theo:PB}, Theorem~\ref{theo:MVPB1} provides a tool to derive PAC-Bayesian generalization bounds for a multiview supervised learning setting. 
Indeed, by making use of the same trick as \citet{GermainLLM09,GermainLLMR15}, the generalization bounds can be derived from Theorem~\ref{theo:MVPB1} by choosing a suitable convex function~$D$ and upper-bounding 
$\Esp_S \Esp_v \Esp_h e^{ m\, D(R_S(h) , R_\D(h))} $. 
We provide the specialization to the three most popular PAC-Bayesian approaches \cite{McAllester99,catoni2007pac,Seeger02,Langford05} in the next section.

Following the same approach, we can obtain a mutiview bound for the expected disagreement.
\begin{theorem}
	\label{theo:MVPB1d}
	Let $V \geq 2$ be the number of views. 
	For any distribution $\D$ on $\Xcal \times \Ycal$, for any set of prior distributions $\{\prior_v\}_{v=1}^V$, for any hyper-prior distribution $\hyperprior$ over $\Vcal$, for any convex function $D : [0,1]  \times  [0,1]  \to \mathbb{R}$,
	we have
	\begin{align*}
	&D\Big(  \E{ S \sim \D^m}     \dmvsS, \E{ S \sim \D^m} \dmvS\Big)\\ 
	\leq   \ &\frac{2}{m} \bigg[  \E{ S \sim \D^m}\E{ v \sim \hyperpostS}      \KL(\postvS\| \prior_v)
	+   \ \Eclap{S \sim \D^m} \ \KL(\hyperpostS \| \hyperprior)
	 + \ln \sqrt{   \E{S \sim \D^m}\E{(h,h') \sim \prior^2} \hspace{-3mm} e^{mD\left(d_S(h,h'), d_\D(h,h') \right)}  }\bigg]\, .&
	\end{align*}
\end{theorem}
\begin{proof}
	The result is obtained straightforwardly by following the proof steps of Theorem~\ref{theo:MVPB1}, using the disagreement instead of the Gibbs risk. Then, similarly at what we have done to obtain Theorem~\ref{theo:PBd}, we substitute  
	$\KL(\postvStwo\| \prior^2_v)$ 	by $2\KL(\postvS\| \prior_v)$,	and $\KL(\hyperpostS^2 \| \hyperprior^2)$  by $2\KL(\hyperpostS \| \hyperprior)$.
\end{proof}

\subsection{Specialization of our Theorem to the Classical Approaches}
In this section, we provide specialization of our multiview theorem to the most popular PAC-Bayesian approaches~\citep{McAllester99,catoni2007pac,Seeger02,Langford05}. To do so, we follow the same principles as~\citet{GermainLLM09,GermainLLMR15}.
\subsubsection{A McAllester-Like Theorem}
\label{sec:mcallester}
We derive here the  specialization of our multiview PAC-Bayesian theorem to the \citet{McAllester03}'s point of view.
\begin{cor}\label{cor:mcallester}
	Let $V \geq 2$ be the number of views. 
	For any distribution $\D$ on $\Xcal \times \Ycal$, for any set of prior distributions $\{\prior_{v}\}_{v=1}^V$, for any hyper-prior distribution $\hyperprior$ over $\Vcal$, we have
\begin{align*}
\E{ S \sim \D^m} R_{\mathcal{D}}(G_{\hyperpostS}^\mv)  \leq \frac{1}{2} \E{ S \sim \D^m} \dmvsS +\!\! \!  \E{ S \sim \D^m} \emvsS +  \!
\sqrt{\displaystyle \frac{ \E{ S \sim \D^m} \E{ v \sim \hyperpostS}    \KL(\postvS \| \prior_v)  + \!\!\!
\E{ S \sim \D^m}  \KL(\hyperpostS \| \hyperprior)  +  \ln \frac{ 2\sqrt{m}}{\delta}}{2m}}\,. 
\end{align*}
\end{cor}
\begin{proof}
	\label{Prof:Cor1}
	To prove the above  result, we apply Theorem~\ref{theo:MVPB1} with $D(a,b)\, =\, 2(a-b)^2$.\\
Then, we upper-bound $\E{S \sim \D^m}  \E{ v \sim \hyperprior}\ \E{ h \sim \prior_v}   e^{ m\, D(R_S(h) , R_\D(h))}$. 
According to Pinsker's inequality, we have $$D(a,b)\ \leq\  \kl(a,b) \ =\ a\ln\frac{a}{b}+(1-a)\ln\frac{1-a}{1-b}.$$
By considering $R_S(h)$ as a random variable which follows a binomial distribution of $m$ trials with a probability of success $R(h)$, we obtain
		\begin{align*}
\E{ S \sim \D^m}\,  \E{ v \sim \hyperprior}\, \E{ h \sim \prior_v}   e^{ m\, D(R_S(h) , R_\D(h))}\
		\leq\ &\E{ S \sim \D^m}\,  \E{ v \sim \hyperprior}\, \E{ h \sim \prior_v}   e^{ m\, \kl(R_S(h) , R_\D(h))} \\
		 =\ & \E{ v \sim \hyperprior} \, \E{ h \sim \prior_v} \, \E{ S \sim \D^m}   \left[ \frac{R_S(h)}{R_\D(h)}  \right]^{mR_S(h)}  \left[ \frac{1 - R_S(h)}{1 - R_\D(h)}  \right]^{m(1 - R_S(h))} \\
		=\ &  \E{ v \sim \hyperprior}\, \E{ h \sim \prior_v}   \sum_{k=0}^{m}  \underset{S \sim \D^m}{\Pr}  \big[ \,R_S(h) = \tfrac{k}{m} \big]  \left[ \frac{k/m}{R_\D(h)}  \right]^{ k}   
			\left[ \frac{1 - k/m}{1 - R_\D(h)}  \right]^{ m - k} \\
	 = \ & \sum_{k=0}^m  \binom mk  \left[ \frac{k}{m} \right]^k \left[ 1- \frac{k}{m} \right]^{m-k} \\
 \leq\ &  2\sqrt{m}\,.
		\end{align*}
		\vspace{-10mm}
		
\end{proof}

\subsubsection{A Catoni-Like Theorem}
To derive a generalization bound with the \citet{catoni2007pac}'s point of view---given a convex function ${\cal F}$ and a real number $C>0$---we define the measure of deviation between the empirical disagreement/joint error and the true risk as  $D(a,b)=\mathcal{F}(b)-C\,a$~\citep{GermainLLM09,GermainLLMR15}.
We obtain the following generalization bound. 
\begin{cor}\label{cor:catoni}
  Let $V \geq 2$ be the number of views. 
For any distribution $\mathcal{D}$ on $\Xcal \times \Ycal$, for any set of prior distributions $\{\prior_v\}_{v=1}^V$, for any hyper-prior distributions $\hyperprior$ over $\Vcal$,  for all $C>0$, we have:
\begin{align*}
\E{ S \sim \D^m} R_{\mathcal{D}}(G_{\hyperpostS}^\mv)\ & \le\ \frac{1}{1-e^{-C}}\Bigg( 1- \exp  \bigg[-\bigg(C\,\bigg(\frac{1}{2} \E{ S \sim \D^m} \dmvsS + \E{ S \sim \D^m} \emvsS \bigg) + \\
& \qquad\qquad\qquad \qquad \frac{1}{m} \Big[ \E{ S \sim \D^m}\ \E{ v \sim \hyperpostS}  \KL(\postvS \| \prior_v) + \E{ S \sim \D^m} \KL(\hyperpostS \| \hyperprior) + \ln \tfrac{1}{\delta} 	\Big] \bigg) \bigg]\Bigg)
\end{align*}
\end{cor}

\begin{proof}
\label{Prof:Cor3}
The result comes from Theorem~\ref{theo:MVPB1}  by taking $D(a,b) = \mathcal{F}(b) - C a$, for a convex $\cal F$ and $C > 0$, and by upper-bounding $\E{ S \sim \D^m}  \E{ v \sim \hyperprior} \E{ h \sim \prior_v} e^{ m D(R_S(h) , R_\D(h))}$. We consider $R_S(h)$ as a random variable following a binomial distribution of $m$ trials with a probability of success $R(h)$. We have:
\begin{align*}
 \E{ S \sim \D^m} \,  \E{ v \sim \hyperprior}\, \E{ h \sim \prior_v} e^{ m\, D(R_S(h) , R_\D(h))} \
& =\  \E{ S \sim \D^m} \,  \E{ v \sim \hyperprior} \,\E{ h \sim \prior_v} e^{ m\, \mathcal{F} (R_\D(h)- C\,m\, R_S(h))} \\
&=\   \E{ S \sim \D^m}   \E{ v \sim \hyperprior} \E{ h \sim \prior_v}    e^{ m\, \mathcal{F} (R_\D(h))} \sum_{k=0}^m \underset{S \sim (\D)^m}{\Pr}  \left( R_S(h) = \frac{k}{m} \right) e^{ -Ck} \\
& =\      \E{ S \sim \D^{ m}}   \E{ v \sim \hyperprior}\, \E{ h \sim \prior_v}    e^{ m \mathcal{F}  (R_\D(h)) } \sum_{k=0}^m  {\textstyle \binom{m}{k}} R_\D(h)^k (1 - R_\D(h))^{m - k } e^{-Ck} \\
&=\   \E{ S \sim \D^m} \,  \E{ v \sim \hyperprior} \,\E{ h \sim \prior_v}  e^{ m \mathcal{F} (R_\D(h))} \big(R_\D(h)\,e^{- \,C}  + (1 - R_\D(h))\big)^m  .
\end{align*}
The corollary is obtained with $$\mathcal{F}(p) = \ln \frac{1}{(1 - p[1 - e^{-C}])}.$$ 
\end{proof}

\subsubsection{A Langford/Seeger-Like Theorem.}
If we make use, as function $D(a,b)$ between the empirical risk and the true risk, of the Kullback-Leibler divergence between two Bernoulli distributions with probability of success $a$ and $b$, we can obtain a bound similar to \cite{Seeger02,Langford05}.
Concretely, we apply Theorem~\ref{theo:MVPB1} with:
\begin{align*}
D(a,b)\ =\ \kl(a,b)\ =\ a\ln\frac{a}{b}+(1-a)\ln\frac{1-a}{1-b}.
\end{align*}
\begin{cor}\label{cor:seeger} Let $V \geq 2$ be the number of views. 
For any distribution $\mathcal{D}$ on $\Xcal \times \Ycal$, for any set of prior distributions $\{\prior_v\}_{v=1}^V$, for any hyper-prior distributions $\hyperprior$ over views $\Vcal$,  we have:
\begin{align*}
&\kl\left(\tfrac12 \E{ S \sim \D^m} \dmvsS + \E{ S \sim \D^m} \emvsS , \E{ S \sim \D^m} R_{\mathcal{D}}(G_{\hyperpostS}^\mv)\right) \\
 \leq \  &\frac{1}{m} \bigg[ \E{ S \sim \D^m}  \E{ v \sim \hyperpostS}  \KL(\postvS \| \prior_v) + \E{ S \sim \D^m}   \KL(\hyperpostS \| \hyperprior) + \ln \frac{ 2\sqrt{m}}{\delta}	\bigg].
\end{align*}
where $\displaystyle \xi(m)=\sum_{k=0}^m  \binom mk  \bigg( \frac{k}{m} \bigg)^k \bigg( 1- \frac{k}{m} \bigg)^{m-k}    \leq  2\sqrt{m}$.
\end{cor}
\begin{proof}
\label{Prof:Cor1}
The result follows from Theorem~\ref{theo:MVPB1}  by taking $D(a,b)=\kl(a,b)$, and upper-bounding $\E{S \sim \D^m}\,  \E{ v \sim \hyperprior} \E{ h \sim \prior_v}   e^{ m\, \kl(R_S(h) , R_\D(h))}$. By considering $R_S(h)$ as a random variable which follows a binomial distribution of $m$ trials with a probability of success $R(h)$, we can prove:
\begin{align*}
 \E{ S \sim \D^m}\,  \E{ v \sim \hyperprior}\, \E{ h \sim \prior_v}   e^{ m\, \kl(R_S(h) , R_\D(h))} \ & = \ \E{ v \sim \hyperprior} \, \E{ h \sim \prior_v} \, \E{ S \sim \D^m}   \left[ \frac{R_S(h)}{R_\D(h)}  \right]^{mR_S(h)}  \left[ \frac{1 - R_S(h)}{1 - R_\D(h)}  \right]^{m(1 - R_S(h))} \\
&   = \ \E{ v \sim \hyperprior}\, \E{ h \sim \prior_v}   \sum_{k=0}^{m}  \underset{S \sim \D^m}{\Pr}   \left( R_S(h) = \tfrac{k}{m} \right)  \left[ \frac{k/m}{R_\D(h)}  \right]^{ k}   
 \left[ \frac{1 - k/m}{1 - R_\D(h)}  \right]^{ m - k}\\
& = \ \sum_{k=0}^m  \binom mk  \left[ \frac{k}{m} \right]^k \left[ 1- \frac{k}{m} \right]^{m-k}\\ &= \ \xi(m).
\end{align*}
\end{proof}

	\section{Discussion on Related Work}
\label{sec:ComparativeStudy}
In this section, we discuss two related theoretical studies of multiview learning related to the notion of Gibbs classifier.

\citet{Massih09} proposed a Rademacher analysis of the risk of the stochastic Gibbs classifier over the view-specific models (for more than two views) where the distribution over the views is restricted to the uniform distribution. 
In their work, each view-specific model is found by minimizing the empirical risk: $\displaystyle h_v^* \ =\ \argmin_{h \in \mathcal{H}_v} \frac{1}{m}\sum_{(\mathbf{x},y) \in S} \Ind{[h(x^v) \neq y]}.$
The prediction for a multiview example $\mathbf{x}$ is then based over the stochastic Gibbs classifier defined according to the uniform distribution, {\it i.e.,} $\forall v\in V,\ \hyperposterior(v)=\frac1V$. The risk of the multiview classifier Gibbs  is hence given by
\begin{align*}
R_{\mathcal{D}}(G_{\rho={1/V}}^{\mv}) = \E{(\mathbf{x},y) \sim \mathcal{D}} \  \frac1V \sum_{v=1}^V \Ind{[h_v^*(x^v) \neq y]}.
\end{align*}

Moreover, \citet{SunS14} proposed a PAC-Bayesian analysis for multiview learning over the concatenation of the views, where the number of views is set to two, and deduced a SVM-like learning algorithm from this framework. The key idea of their approach is to define a prior distribution that promotes similar classification among the two views, and the notion of diversity among the views is handled by a different strategy than ours.
We believe that the two approaches are complementary, as in the general case of more than two views that we consider in our work, we can also use a similar informative prior as the one proposed by \citet{SunS14} for learning.

	\section{Experiments}
\label{sec:experiments}
\label{sec:expe}
\begin{table*}[t]	
	\Huge
	\setlength{\tabcolsep}{3pt}
	\caption{Accuracy and \mbox{$F1$-score} averages for all the classes over $20$ random sets. Note that the results are obtained for different sizes $m$ of the learning sample and are averaged over the six {\it one-vs-all} classification problems.
	 Along the columns, best results are in bold. $^{\downarrow}$ indicates statistically significantly worse performance than the best result, according to Wilcoxon rank sum test ($p < 0.02$)~\citep{StatsRankMethods}.
	 }
	\label{tab:avg_results_all_classifiers}
	\centering
	\resizebox{\textwidth}{!}{
		\begin{tabular}{ c| c c c c  c c c c c c c} 
			\hline
			\multirow{2}{*}{Strategy} &
			\multicolumn{2}{c}{$m=150$} &&
			\multicolumn{2}{c}{$m=200$} && 
			\multicolumn{2}{c}{$m=250$} &&
			\multicolumn{2}{c}{$m=300$} \\
			\cline{2-3}\cline{5-6}\cline{8-9}  \cline{11-12}
			& Accuracy & $F_1$ && Accuracy & $F_1$ && Accuracy & $F_1$  && Accuracy & $F_1$ \\\hline
			$\mono$  & $.8516 \!\pm\!.0031^{\downarrow}$ &  $.1863 \!\pm\! .0299 ^{\downarrow}$ &&  $.8424 \!\pm\! .0272^{\downarrow}$  & $.3056\!\pm\! .0233^{\downarrow}$ && $.8691\!\pm\! .0017^{\downarrow}$ &  $.3352\!\pm\! .0164^{\downarrow}$  && $.8770\!\pm\! .0018^{\downarrow}$ &  $.4103\!\pm\! .0158^{\downarrow}$\\
			
			$\concatSVM$ & $.8507 \!\pm\! .0051^{\downarrow}$ &  $.1577 \!\pm\! .0403^{\downarrow}$ &&  $.8615\!\pm\! .0018^{\downarrow}$  & $.2505\!\pm\! .0182^{\downarrow}$ && $.8674\!\pm\! .0026^{\downarrow}$ &  $.3006\!\pm\! .0267^{\downarrow}$  && $.8746\!\pm\! .0022^{\downarrow}$ &  $.3647\!\pm\! .0258^{\downarrow}$\\[2mm]
			\hdashline 
			\\
			$\AggregU$ & $.8521\!\pm\! .0041^{\downarrow}$ &  $.1810 \!\pm\! .0305^{\downarrow}$ &&  $.8420\!\pm\! .0385^{\downarrow}$  & $.2852\!\pm\! .0339^{\downarrow}$ && $.8676\!\pm\! .0023^{\downarrow}$ &  $.3027\!\pm\! .0234^{\downarrow}$  && $.8774 \!\pm\! .0021^{\downarrow}$ &  $.3945\!\pm\! .0185^{\downarrow}$\\
			
			$\AggregUL$ & $.8507 \!\pm\! .0043^{\downarrow}$ &  $.1653 \!\pm\! .0336^{\downarrow}$ &&  $.8477\!\pm\! .0377^{\downarrow}$  & $.2806\!\pm\! .0244^{\downarrow}$ && $.8682\!\pm\! .0022^{\downarrow}$ &  $.3116\!\pm\! .0210^{\downarrow}$  && $.8773\!\pm\! .0024^{\downarrow}$ &  $.3943\!\pm\! .0204^{\downarrow}$\\[2mm]
			\hdashline 
			\\
			
			$\AggregSVMall$  & $.8568 \!\pm\! .0087^{\downarrow}$ &  $.3899 \!\pm\! .0789^{\downarrow}$ &&  $.8527 \!\pm\! .0406^{\downarrow}$  & $.5027 \!\pm\! .0780$ && $.8490 \!\pm\! .0716^{\downarrow}$ &  $\textbf{.5399} \!\pm\! \textbf{.0585}$  && $.8422 \!\pm\! .0526^{\downarrow}$ &  $\textbf{.5779} \!\pm\! \textbf{.0422}$\\

			$\AggregCqall$  & $\textbf{{.8692}}  \!\pm\! \textbf{.0059}$ &  $\textbf{{.4298}} \!\pm\! \textbf{.0570}$ &&  $\textbf{.8768} \!\pm\! \textbf{.0082}$  & $\textbf{.5066} \!\pm\! \textbf{.0402}$ && $\textbf{.8846 }\!\pm\! \textbf{.0047}$ &  $.5365 \!\pm\! .0371$  && $\textbf{.8881} \!\pm\! \textbf{ .0060}$ &  $.5705  \!\pm\! .0286$\\
			
			\hline
			
		\end{tabular}}
		
	\end{table*}

In this section, we present experiments to highlight the usefulness of our theoretical analysis by following a two-level hierarchy strategy. 
To do so, we learn a multiview model in two stages by following a classifier late fusion approach~\citep{Early-Late-ACMMultimedia05}  (sometimes referred as stacking~\citep{Wolpert92}). 
Concretely, we first learn view-specific classifiers for each view at the base level of the hierarchy. Each view-specific classifier is expressed as a majority vote of kernel functions. Then, we learn weighted combination based on predictions of view-specific classifiers.
It is worth noting that this is the procedure followed by \citet{MorvantHA14} in a PAC-Bayesian fashion, but without any theoretical justifications and in a ranking setting.

We consider a publicly available multilingual multiview text categorization corpus extracted from the Reuters RCV1/RCV2 corpus~\citep{Massih09}\footnote{\url{https://archive.ics.uci.edu/ml/datasets/Reuters+RCV1+RCV2+Multilingual,+Multiview+Text+Categorization+Test+collection}}, which contains more than $110,000$ documents from five different languages (English, German, French, Italian, Spanish) distributed over six classes. 
To transform the dataset into a binary classification task, we consider six {\it one-versus-all} classification problems: For each class, we learn a multiview binary classification model by considering all documents from that class as positive examples and all others as negative examples. 
We then split the dataset into training and testing sets: we reserve a test sample containing $30 \%$ of total documents.
In order to highlight the benefits of the information brought by multiple views, we train the models with small learning sets by randomly choosing the learning sample $S$ from the remaining set of the documents; the number of learning examples $m$ considered are: $150$, $200$, $250$ and $300$. 
For each fusion-based approach, we split the learning sample $S$ into two parts: $S_1$ for learning the view-specific classifier at the first level and $S_2$ for learning the final multiview model at the second level; such that $|S_1|=\frac{3}{5} m$ and $|S_2|=\frac{2}{5} m$ (with $m=|S|$).
In addition, the reported results are averaged on $20$ runs of experiments, each run being done with a new random learning sample. Since the classes are highly unbalanced, we report in Table~\ref{tab:avg_results_all_classifiers} the accuracy along with the \mbox{$F1$-measure}, which is the harmonic average of precision and recall, computed on the test sample.

To assess that multiview learning with late fusion makes sense for our task, we consider as baselines the four following one-step learning algorithms (provided with the learning sample $S$).
First, we learn a view-specific model on each view and report, as $\mono$, their average performance. 
We also follow an early fusion procedure, referred as $\concatSVM$, consisting of learning one single model using SVM~\citep{cortes1995support} over the simple concatenation of the features of five views.
Moreover, we look at two simple voters' combinations, respectively denoted by $\AggregU$ and $\AggregUL$, for which the weights associated with each view follow the uniform distribution. 
Concretely, $\AggregU$, respectively $\AggregUL$, combines the real-valued prediction, respectively the labels, returned by the view-specific classifiers. 
In other words, we have
\begin{align*} 
&\AggregU(\xbf)\ =\ \frac{1}{5} \sum_{v=1}^5  h^v(x^v)\,,\\
 \mbox{ and }\quad 
&\AggregUL(\xbf)\ =\ \frac{1}{5} \sum_{v=1}^5 \sign\left[h^v(x^v)\right]\,,
\end{align*}
where $ h^v(x^v)$ is the real-valued prediction of the view-specific classifier learned on view~$v$. 

We compare the above one-step methods to the two following late fusion approaches that only differ at the second level.
Concretely, at the first level we construct from $S_1$ different view-specific majority vote expressed as linear SVM models\footnote{We use linear SVM model as it is usually done for text classification tasks~\citep[\eg,][]{Joachims:1998}.} with different hyperparameter $C$ values ($12$ values between $10^{-8}$ and $10^{3}$): We do not perform cross-validation at the first level. This has the advantage to {\it (i)} lighten the first level learning process, since we do not need to validate models, and {\it (ii)} to potentially increase the expressivity of the final model.

At the second level, as it is often done for late fusion, we learn from $S_2$  the final weighted combination over the view specific voters using a RBF kernel.
The methods referred as $\AggregSVMall$, respectively $\AggregCqall$, make use of SVM, respectively the PAC-Bayesian algorithm CqBoost~\citep{cqboost}.
Note that, as recalled in Section~\ref{sec:PAC-Bayes}, CqBoost is an algorithm that tends to minimize the C-Bound of Equation~\eqref{eq:cbound}: it directly captures a trade-off between accuracy and disagreement.

We follow a $5$-fold cross-validation procedure for selecting the hyperparameters of each learning algorithm.
For $\mono$, $\concatSVM$, $\AggregU$ and $\AggregUL$ the hyperparameter $C$ is chosen over a set of $12$ values between $10^{-8}$ and $10^{3}$.
For $\AggregSVMall$ and $\AggregCqall$ the  hyperparameter $\gamma$ of the RBF kernel is chosen over $9$ values between $10^{-6}$  and $10^{2}$.
For $\AggregSVMall$, the hyperparameter $C$ is chosen over a set of $12$ values between $10^{-8}$  and $10^{3}$. For $\AggregCqall$, the hyperparameter $\mu$ is chosen over a set of $8$ values between $10^{-8}$  and $10^{-1}$. 
Note that we made use of the \emph{scikit-learn}~\citep{scikit-learn} implementation for learning our SVM models.

First of all, from Table~\ref{tab:avg_results_all_classifiers}, the two-step approaches provide the best results on average.
Secondly, according to a Wilcoxon rank sum test~\citep{StatsRankMethods} with $p < 0.02$, the PAC-Bayesian late fusion based approach $\AggregCqall$ is significantly the best method---in terms of accuracy, and except for the smallest learning sample size ($m=150$), $\AggregCqall$ and $\AggregSVMall$ produce models with similar $F1$-measure.
We can also remark that $\AggregCqall$ is more ``stable'' than $\AggregSVMall$ according to the standard deviation values.
These results confirm the potential of using PAC-Bayesian approaches for multiview learning where we can control a trade-off between accuracy and diversity among voters.

	\section{Conclusion and Future Work}
\label{sec:conclu}
In this paper, we proposed a first PAC-Bayesian analysis of weighted majority vote classifiers for multiview learning when observations are described by more than two views.  
Our analysis is based on a hierarchy of distributions, {\it i.e.} weights, over the views and voters: {\it(i)} for each view $v$ a posterior and prior distributions over the view-specific voter's set, and {\it (ii)} a hyper-posterior and hyper-prior distribution over the set of views. 
We derived a general PAC-Bayesian theorem tailored for this setting, that can be specialized to any convex function to compare the empirical and true risks of the stochastic Gibbs classifier associated with the weighted majority vote.
We also presented a similar theorem for the expected disagreement, a notion that turns out to be crucial in multiview learning.
Moreover, while usual PAC-Bayesian analyses are expressed as probabilistic bounds over the random choice of the learning sample, we presented here bounds in expectation over the data, which is very interesting from a PAC-Bayesian standpoint where the posterior distribution is data dependent.

According to the distributions' hierarchy, we evaluated a simple two-step learning algorithm (based on late fusion) on a multiview benchmark.
We compared the accuracies while using SVM and the PAC-Bayesian algorithm CqBoost for weighting the view-specific classifiers. 
The latter revealed itself as a better strategy, 
as it deals nicely with accuracy and the disagreement trade-off promoted by our PAC-Bayesian analysis of the multiview hierarchical approach.

We believe that our theoretical and empirical results are a first step toward the goal of theoretically understanding the multiview learning issue through the PAC-Bayesian point of view, and toward the objective of deriving new multiview learning algorithms. It gives rise to exciting perspectives.\\
Among them, we would like to specialize our result to linear classifiers for which PAC-Bayesian approaches are known to lead to tight bounds and efficient learning algorithms~\citep{GermainLLM09}. This clearly opens the door to derive theoretically founded algorithms for multiview learning.\\
 Another possible algorithmic direction is to take into account a second statistical moment information to link it explicitly  to important properties between views, such as diversity or agreement \cite{Kuncheva,Massih09}. A first direction is to deal with our multiview PAC-Bayesian C-Bound of Lemma~\ref{lem:mv-cbound}---that already takes into account such a notion of diversity~\cite{MorvantHA14}---in order to derive an algorithm as done in a mono-view setting by~\citet{MinCQ,cqboost}.\\
Another perspective is to extend our bounds to diversity-dependent priors, similarly to the approach used by \citet{SunS14}, but for more than two views. This would allow to additionally consider  an {\it  a priori} knowledge on the diversity.\\
Moreover, we would like to explore the \textit{semi-supervised} multiview learning where one has access to  unlabeled data $S_u=\{\xbf_j\}_{j=1}^{m_u}$ along with labeled data $S_l=\{(\xbf_i,y_i)\}_{i=1}^{m_l}$ during training. Indeed, an interesting behaviour of our theorem is that it can be easily extended to this situation: the bound will be a concatenation of a bound over $\tfrac12 \dmvsu$ (depending on $m_u$) and a bound over $\emvsl$ (depending on $m_s$). The main difference with the supervised bound is that the Kullback-Leibler divergence will be multiplied by a factor $2$.

	\appendix
\section*{Appendix---Mathematical Tools}
\label{sec:Appendix}

\begin{theorem}[Markov's ineq.]
	\label{theo:markov}
	For any random variable $X$ {\it s.t.} $\mathbb{E}(|X|) =  \mu$, for any $a > 0$, we  have $$\displaystyle \mathbb{P}(|X| \ge a) \le \frac{\mu}{a}.$$
\end{theorem}

\begin{theorem}[Jensen's ineq.]
\label{theo:jensen}
For any random variable $X$, for any concave function $g$, we have $$\displaystyle g(\Esp [X]) \ \ge\  \Esp [g(X)].$$
\end{theorem}

\begin{theorem}[Cantelli-Chebyshev ineq.]
\label{theo:chebyshev}
For any random variable $X$ {\it s.t.} $\mathbb{E}(X) =  \mu$ and $\mathbf{Var}(X)=\sigma^2$, and for any $a > 0$, we  have
$$\mathbb{P}(X - \mu \ge a) \le \frac{\sigma^2}{\sigma^2 + a^2}.$$
\end{theorem}

	\subsection*{Acknowledgments.} This work is partially funded by the French ANR project LIVES ANR-15-CE23-0026-03, the ``R\'egion  Rh\^{o}ne-Alpes'', and by the CIFAR program in Learning in Machines \& Brains.
	
	\bibliography{ecml17}
	
\end{document}